\documentclass[11pt, a4paper, copyright, gdm]{google}

\usepackage[authoryear, sort&compress, round]{natbib}

\title{Refined Detection for Gumbel Watermarking}

\author[1]{Tor Lattimore}

\affil[1]{\thepa{}{}}

\uselogo{} 
\correspondingauthor{lattimore@google.com}

\usepackage{tikz}
\usetikzlibrary{positioning}

\usepackage[plain]{algorithm}
\usepackage{thmtools} 
\usepackage{mathtools}
\usepackage{listings}
\usepackage{mdframed}
\usepackage{nicefrac}
\usepackage{bm}
\usepackage{listings}
\usepackage{inconsolata}
\usepackage{setspace}

\newcommand{\keyword}[1]{\texttt{\color{red!80!black} #1}}

\usepackage[capitalize]{cleveref}

\newcommand{\Gl}{G}
\newcommand{\Gh}{G^{\text{\scriptsize  $\nicefrac{1}{2}$}}}
\newcommand{\Go}{G^{\text{\scriptsize $1$}}}

\newcommand{\E}{\mathbb E}

\newcommand{\bbP}{\mathbb P}
\newcommand{\bbQ}{\mathbb Q}

\newcommand{\cF}{\mathcal F}
\newcommand{\sF}{\mathcal F}

\newcommand{\cU}{\mathcal U}

\renewcommand{\d}[1]{\operatorname{d}\!#1}
\newcommand{\eps}{\varepsilon}
\newcommand{\sind}{\bm 1}

\newcommand{\argmax}{\operatornamewithlimits{arg\,max}}

\newcommand{\KL}{\operatorname{KL}}

\newcommand{\exl}[1]{\stackrel{\mathclap{\text{\tiny \texttt{\color{red!60!black}(#1)}}}}}
\newcommand{\ex}[1]{\texttt{\color{red!60!black}(#1)}}

\theoremstyle{plain}
\newtheorem{theorem}{Theorem}

\newtheorem{lemma}[theorem]{Lemma}
\newtheorem{proposition}[theorem]{Proposition}

\theoremstyle{definition}
\newtheorem{definition}[theorem]{Definition}
\newtheorem{assumption}[theorem]{Assumption}

\newtheorem{example}[theorem]{Example}
\newtheorem{remark}[theorem]{Remark}

\crefname{example}{Example}{Examples}

\begin{abstract}
We propose a simple detection mechanism for the Gumbel watermarking scheme proposed by \cite{AAR23}.
The new mechanism is proven to be near-optimal in a problem-dependent sense among all model-agnostic watermarking schemes under the assumption that
the next-token distribution is sampled i.i.d.
\end{abstract}

\begin{document}
\maketitle

\section{Introduction}
We are interested in the problem of watermarking the output of large language models.
The token space is a finite set $\Sigma = \{1,\ldots,d\}$.
Let $\Sigma^*$ be the space of all finite sequences over alphabet $\Sigma$.
A language model is a function $p : \Sigma^* \to \Delta(\Sigma)$ where $\Delta(\Sigma)$ is the space
of probability distributions over the tokens. Equivalently, $p$ is a probability kernel from $\Sigma^*$ to $\Sigma$.
A sequence of tokens $A_1,\ldots,A_n$ is sampled autoregressively from $p$ in the usual way:
\begin{align*}
A_t \sim p(\cdot | A_1,\ldots,A_{t-1}) = P_t \,.
\end{align*}
Watermarking is a mechanism for sampling from the language model (preferably without distortion) in such a way that the owner
of the model can later distinguish those sequences that have been sampled from the model from those that have not.
The Gumbel watermarking scheme proposed by \cite{AAR23} works as follows (see \cref{alg:water}):
\begin{enumerate}
\item Let \texttt{key} be a secret key known to the language model designer.
\item When sampling $A_t$ from the model, use a hash function to produce a random seed
using \texttt{key} and $A_{t-1},\ldots,A_{t-k}$ where $k$ is some small integer. For example, $k = 3$.
\item Use this random seed to generate pseudorandom uniform variables $(U_{t,a})_{a \in \Sigma}$ in $[0,1]^\Sigma$ and
then select 
\begin{align*}
A_t = \argmax_{a \in \Sigma} \frac{P_t(a)}{-\log(U_{t,a})} = \argmax_{a \in \Sigma} \left[-\log(-\log(U_{t,a})) + \log(P_t(a))\right] \,.
\end{align*}
Setting aside the issue that the $(U_{t,a})$ are only pseudorandom, the conditional law of $A_t$ is $P_t$ by the Gumbel trick and recalling that
$-\log(-\log(U))$ is Gumbel distributed when $U$ is uniformly distributed on $[0,1]$.
\end{enumerate}

\begin{algorithm}[h!]
\centering
\begin{tikzpicture}
\node[draw,fill=white!90!black,text width=12cm] (b) at (0,0) {
\begin{lstlisting}
args: key, $k$, model $p$
for $t = 1$ to $\infty$:
  let $P_t = p(\cdot | A_1,\ldots,A_{t-1})$
  let $U_t = \keyword{rand}(\texttt{key}, A_{t-1}, A_{t-2},\ldots,A_{t-k}) \in [0,1]^\Sigma$
  output token $A_t = \argmax_{a \in \Sigma} P_t(a) / (-\log U_{t,a})$
  if $A_t = \keyword{end-of-message}$: break
\end{lstlisting}
};
\node[draw,fill=yellow!90!black] at (b.north) {\textsc{gumbel watermarking}}; 
\end{tikzpicture}
\caption{The watermarking algorithm accepts a secret key and a model and autoregressively samples
a sequence of tokens until the \keyword{end-of-message} token is sampled.
The \keyword{rand} function uses its input to compute a pseudorandom vector $U_t$ uniformly distributed on $[0,1]^\Sigma$.
}
\label{alg:water}
\end{algorithm}

\paragraph{Detection}
Detection works by noting that, when given the secret key,
the sequence of noise random variables $(U_t)$ can be reconstructed from the text $(A_t)$. Then let $V_t = U_{t,A_t}$.
When the text is not watermarked one would expect $V_1,\ldots,V_n$ to behave like an i.i.d.\ sequence of random variables that are uniformly distributed on $[0,1]$.
On the other hand, when the text is watermarked, then $V_t$ is not uniformly distributed unless $P_t$ is a Dirac.
Hence, the detector can identify the watermark by proving statistically that the observed $(V_t)_{t=1}^n$ were not sampled i.i.d.\ from the uniform distribution.
Importantly, detection is possible without access to the model.
The generic detection mechanism is given in \cref{alg:detect-general}.

\begin{algorithm}[h!]

\centering
\begin{tikzpicture}
\node[draw,fill=white!90!black,text width=13cm] (b) at (0,0) {

\begin{lstlisting}
args: key, $k$, confidence level $\delta$ and $A_1,\ldots,A_n$
set $\epsilon = \frac{\log(1/\delta)}{n}$ and compute $\tau_\star$ so that $\cref{eq:opt}$ holds 
for $t = 1$ to $n$:
  let $U_t = \keyword{rand}(\texttt{key}, A_{t-1}, A_{t-2},\ldots,A_{t-k})$
  let $V_t = U_{t,A_t}$ 
if $\keyword{goodness-of-fit-test}((V_t)_{t=1}^n, \delta)$: return $\keyword{watermarked}$
return $\keyword{unknown}$
\end{lstlisting}

};
\node[draw,fill=yellow!90!black] at (b.north) {\textsc{detection}}; 
\end{tikzpicture}
\caption{The detection algorithm accepts the secret key, the confidence level $\delta \in (0,1)$ that controls the
false positive rate and the sequence of tokens to be evaluated.
The \keyword{goodness-of-fit-test} is a well-chosen statistical test that the $(V_t)_{t=1}^n$ are independent and uniformly distributed on $[0,1]$
at confidence level $\delta$. Various options are available as we discuss in \cref{sec:exp} and \cref{sec:power}.
}
\label{alg:detect-general}
\end{algorithm}

\paragraph{Contribution}
We study the Gumbel watermarking procedure (\cref{alg:water}), but replace the detection test proposed by \cite{AAR23} with a new test that is theoretically more statistically efficient.
We prove nearly matching problem-dependent upper and lower bounds on the number of tokens needed to detect watermarked text. 
Our bounds are in terms of an entropy-like quantity of the distribution of next-token distributions.
The lower bound applies to all watermarking schemes in the same framework (to be explained) and hence the Gumbel scheme combined with the new
detector is (in a certain sense) nearly optimal.

\paragraph{Related work}
There are now a huge variety of schemes for watermarking language models, with a range of desiderata
and analysis objectives \citep{liu2024survey,piet2025markmywords}.
A relatively recent summary of the current state-of-the-art can be found in the review by \cite{ji2025overview}, but the field is growing incredibly rapidly.
Some schemes are distortionary in the sense that they change the distribution of the model \citep{kirchenbauer23a,zhao2023provable,liu2023unforgeable} while
others are non-distortionary \citep{AAR23,dathathri2024scalable}.
Our work is based on the non-distortionary Gumbel watermarking scheme proposed and analysed by \cite{AAR23}.
\cite{AAR23} used a detection mechanism based on a different statistic than the one proposed here and proved that
detection occurs with high probability once $n = \Omega(\log(1/\delta) / \bar H^2)$ with $\bar H$ the average entropy
(see \cref{eq:Hn}).
This is the detection time that seems to be obtained most naturally when the concentration analysis used is
not variance-aware \citep{christ2024undetectable}.
\cite{huang2023towards} theoretically study watermarking procedures that are both distortion free and
distortionary. Perhaps the most relevant of their results is a proof that when $P_t = p$ is constant, then
detection is possible if $n = \Omega(\log(1/H(p))\log(1/\delta) / H(p))$ and that this is optimal. But this result is not obtained
using a simple practical watermarking procedure.
\cite{tsur2025optimized} ambitiously tackle the problem of designing exactly optimal statistical watermarking/detection schemes
but focus principally on the single-shot case where only a single token is sampled from the model, with preliminary results
in the sequential setting with a constant token distribution ($P_t = p\, \forall t$).
Like us, \cite{li2025statistical} also examine the problem of refining the detection of Gumbel watermarking (and the scheme by \cite{kuditipudi2023robust}). They propose a model-agnostic detector that is optimal in an asymptotic minimax framework where $P_t$ is assumed to
satisfy $\max_{a \in \Sigma} P_t(a) \leq 1 - \Delta$ for some user-provided constant $\Delta \in (0,1)$.
On the one hand, they are able to establish a kind of exact optimality. On the other hand, thanks to the minimax nature of their 
optimality definition, it is unclear if their method achieves the same kind of near-optimal problem-dependent detection efficiency
as our proposal.
Our statistical test is inspired by the classical literature on goodness-of-fit tests, and particularly the Anderson--Darling test \citep{AD52}.
Goodness-of-fit tests have been used for Gumbel watermarking by \cite{he2025empirical}.
\cite{li2025robust} also use goodness-of-fit tests for detection. They primarily focus on a more complicated version of the problem where some of the text has been generated by an LLM
but other parts are human-generated. Nevertheless, they prove a kind of optimality in certain scenarios. At the same time, their results seem incomparable to ours.

\paragraph{Notation}
The sequence length $n > 1$ and confidence level $\delta \in (0,1)$ are fixed throughout.
When $\mu$ and $\nu$ are probability measures on the same space and $\mu$ is absolutely continuous with respect to $\nu$, then
the relative entropy and $\chi^2$-divergence between $\mu$ and $\nu$ are
\begin{align*}
\KL(\mu, \nu) = \int \log\left(\frac{\d{\mu}}{\d{\nu}}(x)\right) \d{\mu}(x) \quad \text{and} \quad
\chi^2(\mu, \nu) = \int \left(\frac{\d{\mu}}{\d{\nu}}(x) - 1\right)^2 \d{\nu}(x) \,.
\end{align*}
The uniform probability measure on $[0,1]$ is $\cU([0,1])$, which is also denoted by $\lambda$.
The $\sigma$-algebra generated by $(U_s)_{s=1}^t$ and $(A_s)_{s=1}^t$ is
$\cF_t = \sigma(U_1,A_1,\ldots,U_t,A_t)$.

\paragraph{Null and alternative hypotheses}
Our results on the correctness and detection time for the Gumbel watermarking scheme
depend on two probability measures. Under the null hypothesis $\bbP_0$ we assume that
$(U_t)_{t=1}^n$ and $(A_t)_{t=1}^n$ are independent 
so that $(V_t)_{t=1}^n$ is a sequence of independent uniformly distributed random variables.
On the other hand, under the alternative hypothesis $\bbP_1$ we assume that 
$(U_t)_{t=1}^n$ is a sequence of independent random elements sampled from $\cU([0,1]^\Sigma)$ and
$A_t = \argmax_{a \in \Sigma} P_t(a) / (-\log (U_{t,a}))$ where $P_t = p(\cdot | A_1,\ldots,A_{t-1})$
is the next-token distribution generated by the model.
The expectation with respect to $\bbP_k$ with $k \in \{0,1\}$ is denoted by $\E_k$.
These assumptions are not unusual for the analysis of watermarking schemes \citep{huang2023towards,li2025statistical,dathathri2024scalable}.
Care is needed, however, since the statistical assumptions are not strictly-speaking correct because
the $(U_t)_{t=1}^n$ are generated using pseudorandomness and with seeds constructed from the observed tokens.
The statistical assumption could in principle be justified using cryptographic arguments and with
certain modifications to the procedure for generating $(U_t)$. 
Empirical evidence suggests that
the distortion due to pseudorandomness does not impact model performance on downstream tasks \citep{dathathri2024scalable}.

\paragraph{Entropy-like quantities}
Given a distribution $p \in \Delta(\Sigma)$ and $\epsilon \in (0,1)$ let
$p_\epsilon(a) = p(a) \sind(p(a) \leq \epsilon)$ and $p_\epsilon^c(a) = p(a) \sind(p(a) > \epsilon)$.
The entropy of a distribution $p \in \Delta(\Sigma)$ is defined by 
$H(p) = -\sum_{a \in \Sigma} p(a) \log p(a)$
with $0 \log 0 = 0$.
We also need another entropy-like quantity, defined for $q \in (0,1)$ by
\begin{align*}
G^q(p) = \sum_{a \in \Sigma} \left(p(a)(1 - p(a))\right)^q \,.
\end{align*}
Lastly, define $G_\epsilon(p) = \Go(p_\epsilon) / \sqrt{\epsilon} + \Gh(p_\epsilon^c)$.

\section{Exponential detection}\label{sec:exp}
The Gumbel scheme was proposed by \cite{AAR23}, who suggested the following simple detection mechanism.
Let $S_t = \log(1/(1-V_t))$.
Under the null hypothesis $\bbP_0$, $(S_t)_{t=1}^n$ is a sequence of independent exponentially distributed random variables.
The detector can proclaim a sequence is watermarked if $\sum_{t=1}^n S_t$ is unusually large relative
to what would be expected under the null hypothesis.
In the language of \cref{alg:detect-general}, 
\begin{align*}
\keyword{goodness-of-fit}\left((V_t)_{t=1}^n, \delta\right) = \sind\left(\sum_{t=1}^n S_t \geq \tau_\star\right) \,,
\end{align*}
where the critical threshold $\tau_\star$ is chosen so that $\bbP_0(\sum_{t=1}^n S_t \geq \tau_\star) = \delta$.
\cite{AAR23} showed that under the alternative hypothesis $\bbP_1$, $S_t$ has a larger mean than under the null and
that detection occurs with probability at least $1 - \delta$ if $n = \Omega(\log(1/\delta) / \bar H^{2})$ where
\begin{align}
\bar H = \frac{1}{n} \sum_{t=1}^n H(P_t) 
\label{eq:Hn}
\end{align}
is the average entropy.

\section{Power law detection}\label{sec:power}
This paper is about an alternative detection method that is statistically more efficient, at least theoretically (\cref{alg:detect}).
The idea is to replace the exponential statistic in \cref{sec:exp} with a truncated power law. The intuition for this choice is given in the discussion (\cref{sec:disc}).
Let $\epsilon \in (0,1)$ be a constant to be chosen later and 
\begin{align}
S(u) = \min \left(\frac{1}{\sqrt{\epsilon}}, \frac{1}{\sqrt{1 - u}}\right) - \mu \,,
\label{eq:S}
\end{align}
where $\mu = 2 - \sqrt{\epsilon}$ is chosen so that $\E[S(U)] = 0$ when $U$ has law $\cU([0,1])$. 
We will use the statistic $S_t = S(V_t)$, which under the null hypothesis is a truncated power law.
By definition the optimal critical value $\tau_\star$ is chosen so that
\begin{align}
\bbP_0\left(\sum_{t=1}^n S_t \geq \tau_\star\right) = \delta \,.
\label{eq:opt}
\end{align}

\begin{algorithm}[h!]

\centering
\begin{tikzpicture}
\node[draw,fill=white!90!black,text width=13cm] (b) at (0,0) {
\begin{lstlisting}
args: key, $k$, confidence level $\delta$ and $A_1,\ldots,A_n$
set $\epsilon = \frac{\log(1/\delta)}{n}$ and compute $\tau_\star$ so that $\cref{eq:opt}$ holds 
for $t = 1$ to $n$:
  let $U_t = \keyword{rand}(\texttt{key}, A_{t-1}, A_{t-2},\ldots,A_{t-k})$
  let $V_t = U_{t,A_t}$ and $S_t = S(V_t)$
if $\sum_{t=1}^n S_t \geq \tau_\star$: return $\keyword{watermarked}$
return $\keyword{unknown}$
\end{lstlisting}
};
\node[draw,fill=yellow!90!black] at (b.north) {\textsc{power law detection}}; 
\end{tikzpicture}
\caption{The power law detection algorithm accepts the secret key, the confidence level $\delta \in (0,1)$ that controls the
false positive rate and the sequence of tokens to be evaluated.}
\label{alg:detect}
\end{algorithm}

The definition of $\tau_\star$ ensures the false positive rate is exactly $\delta$. Naturally you can use a conservatively large
choice of $\tau_\star$ and still guarantee the false positive rate is at most $\delta$. Such a choice is given in \cref{lem:null}. 

\begin{remark}
Under the null hypothesis $\bbP_0$, $S_t$ has the same distribution as $S(U)$ where $U$ has law $\cU([0,1])$.
Hence $\tau_\star$ can be estimated to high precision using Monte Carlo.
But note that $\tau_\star$ depends on the number of tokens $n$ as well as the confidence level $\delta$ and the choice of $\epsilon$
that defines $S$ in \cref{eq:S}. Let us note in passing that the requirement to know $n$ can be relaxed by using anytime tests, a program that was recently
initiated for watermarking by \cite{huang2026towards}.
The exact asymptotics of $\tau_\star$ are discussed in \cref{sec:disc}.
\end{remark}

What remains is to understand the circumstances under which detection algorithm successfully identifies
watermarked text.

\begin{theorem}\label{thm:upper}
If \cref{alg:detect} is run with input $(A_t)_{t=1}^n$ sampled using \cref{alg:water} and using the same key 
and $\delta \in (0,1)$ and $\epsilon = \log(1/\delta) / n$, then 
\begin{align*}
\bbP_1\left(\sum_{t=1}^n G_\epsilon(P_t) \geq C \tau \text{ and \cref{alg:detect} returns } \keyword{unknown}\right)
\leq 2\delta \,,
\end{align*}
where $C > 0$ is a universal constant and $\tau = \Theta\big(\sqrt{n \log(n) \log(1/\delta)}\big)$ is defined in \cref{lem:null}.
\end{theorem}

In order to prove \cref{thm:upper}
we need to analyse the distribution of the test statistic $\sum_{t=1}^n S_t$ under the alternative hypothesis $\bbP_1$. In particular,
we will show that with high probability it scales proportionally to $\sum_{t=1}^n G_\epsilon(P_t)$.
Beyond this, we also need to argue that $\tau_\star$ is not too large.
To begin the analysis, we study the (conditional) law of $S_t$ and $V_t$ under 
both the null and alternative hypotheses. 
For completeness, a proof of the following lemma is provided in the appendix.

\begin{lemma}[\citealt{fernandez2023three}\label{lem:cdf}]
Suppose that $p \in \Delta(\Sigma)$ and $(U_a)_{a \in \Sigma}$ are independent standard uniform distributions and 
\begin{align*}
A = \argmax_{a \in \Sigma} \frac{p(a)}{-\log(U_a)} 
\end{align*}
and $V = U_A$. Then $\bbP(V \leq x) = \sum_{a \in \Sigma} p(a) x^{1/p(a)} \leq x$.
\end{lemma}

Recall that by definition if $U$ has law $\cU([0,1])$, then $\E[S(U)] = 0$.
The next lemma provides a lower bound on $\E[S(V)]$ when $V$ is defined as in \cref{lem:cdf}.

\begin{lemma}\label{lem:diff}
Under the same assumptions as \cref{lem:cdf} and provided $\epsilon \leq 1/4$, 
\begin{align*}
\E\left[S(V)\right] \geq c \left[\Gl_\epsilon(p) - \sqrt{\epsilon}\right] \,.
\end{align*}
where $c > 0$ is an absolute constant and can be taken to be $c = \frac{1}{240}$.
\end{lemma}

\begin{proof}
We proceed in two steps.

\paragraph{Step 1: Setup and main argument}
Recall that $\mu = 2 - \sqrt{\epsilon}$. 
By \cref{eq:S} and \cref{lem:cdf},
\begin{align}
\E[S(V)]
%&= \E\left[\min\left(\frac{1}{\sqrt{1 - V}}, \frac{1}{\sqrt{\epsilon}}\right)\right] - \mu 
= \int_0^1 \sum_{a \in \Sigma} x^{1/p(a) - 1}\min\left(\frac{1}{\sqrt{1-x}}, \frac{1}{\sqrt{\epsilon}}\right) \d{x} - \mu 
= \sum_{a \in \Sigma} \delta_\epsilon(p(a))\,,
\label{eq:S-delta}
\end{align}
where for $q \in (0,1)$,
\begin{align*}
\delta_\epsilon(q) 
&\triangleq \int_0^1 x^{1/q-1} \min\left(\frac{1}{\sqrt{1-x}}, \frac{1}{\sqrt{\epsilon}}\right) \d{x} - q \mu \,.
\end{align*}
A case-by-case analysis shows that if $\epsilon \in (0,1/4)$ and $q \leq 1/2$, then
\begin{align}
\delta_\epsilon(q) \geq \frac{1}{120} \min\left(\sqrt{q}, q/\sqrt{\epsilon}\right) \,.
\label{eq:c-by-c}
\end{align}
Suppose that $p(a) > 1/2$, then
\begin{align}
\sqrt{p(a)(1 - p(a))}
&\leq \sqrt{1 - p(a)} 
= \sqrt{\sum_{b \neq a} p(b)} 
\leq \sum_{b \neq a} \min\left(\sqrt{p(b)}, \frac{p(b)}{\sqrt{\epsilon}}\right) + \frac{\sqrt{\epsilon}}{4} \,,
\label{eq:edge}
\end{align}
where the final inequality follows using the fact that for $x, y > 0$, $\sqrt{x + y} \leq \sqrt{x} + \sqrt{y}$ and
$\sqrt{x} \leq x/(2y) + y/2$.
Let $M = \{a : p(a) > 1/2\}$, which of course has $|M| \in \{0, 1\}$.
Hence
\begin{align*}
G_\epsilon(p)
&\leq \sum_{a \in M} \sqrt{p(a)(1 - p(a))} + \sum_{b \in M^c} \min(\sqrt{p(b)}, p(b) / \sqrt{\epsilon})  \\
\tag*{by \cref{eq:edge}}
&\leq 2 \sum_{b \in M^c} \min(\sqrt{p(b)}, p(b) / \sqrt{\epsilon}) + \frac{\sqrt{\epsilon}}{4} \\
\tag*{by \cref{eq:c-by-c}}
&\leq 240 \sum_{b \in M^c} \delta_\epsilon(p(b)) + \frac{\sqrt{\epsilon}}{4} \\ 
\tag*{by \cref{eq:S-delta}}
&\leq 240 \E[S(V)] + \frac{\sqrt{\epsilon}}{4} \,. 
\end{align*}
Rearranging completes the proof.

\paragraph{Step 2: Case-by-case argument}
We finish the proof by confirming \cref{eq:c-by-c}.
The function $q \mapsto \delta_\epsilon(q)$ is concave and $\delta_\epsilon(0) = 0$ and for $\epsilon \in (0,1/4)$,
$\delta_\epsilon(1/2) \geq 1/10$.
Therefore $\delta_\epsilon(q) \geq 2q \delta_\epsilon(1/2) \geq q/5$.
$q \geq 1/192$ we have $\delta_\epsilon(q) \geq q/5 \geq \frac{1}{70} \sqrt{q}$. Moreover, if
$\sqrt{\epsilon} \geq 1/24$, then $\delta_\epsilon(q) \geq q/5 \geq \frac{1}{120} q / \sqrt{\epsilon}$.
Suppose now that $q < 1/192$ and $\sqrt{\epsilon} < 1/24$.
Let $y$ be such that $y^{1/q - 1} = 1/2$, which satisfies $y = (1/2)^{q/(1-q)} = \exp(-\frac{q}{1-q} \log(2)) \leq 1 - \frac{q \log(2)}{2(1-q)} \leq 1 - q/3$.
Then
\begin{align*}
\delta_\epsilon(q) 
&= \int_0^{1-\epsilon} \frac{x^{1/q-1} - q}{\sqrt{1 - x}} + \sqrt{\frac{1}{\epsilon}} \int_{1-\epsilon}^1 (x^{1/q-1} - q) \d{x} 
= \int_0^1 x^{1/q-1} \min\left(\frac{1}{\sqrt{1-x}}, \frac{1}{\sqrt{\epsilon}}\right) + q\left(\sqrt{\epsilon} - 2\right) \\
&\exl{a}\geq \frac{1}{2} \min\left(\sqrt{1-y}, \frac{1-y}{\sqrt{\epsilon}}\right) - 2q 
\exl{b}\geq \frac{1}{2} \min\left(\sqrt{q/3}, \frac{q}{3 \sqrt{\epsilon}}\right) - 2q 
\exl{c}\geq \frac{1}{4} \min\left(\sqrt{q/3}, \frac{q}{3 \sqrt{\epsilon}}\right) \,.
\end{align*}
where \ex{a} follows since the integrand is increasing, \ex{b} since $y \leq 1 - q/3$ and \ex{c} since $q < 1/192$ and $\sqrt{\epsilon} < 1/24$.
\end{proof}

The critical value $\tau_\star$ defined in \cref{eq:opt} can be upper bounded using a concentration of measure argument.

\begin{lemma}\label{lem:null}
Suppose that $\tau_\star$ satisfies \cref{eq:opt} for some $\delta \in (0,1)$. Then
\begin{align}
\tau_\star \leq \sqrt{-2 n \log(\epsilon) \log(1/\delta)} + \frac{3 \log(1/\delta)}{2\sqrt{\epsilon}} \triangleq \tau \,.
\label{eq:tau}
\end{align}
\end{lemma}

\begin{proof}
Under the null hypothesis $\bbP_0$, $(V_t)$ is a sequence of independent uniform random variables.
By definition, $S_t = S(V_t)$ and hence $|S_t| \leq 1/\sqrt{\epsilon}$.
Moreover, an elementary integral shows that $\E_0[S_t^2] \leq - \log \epsilon$.
Hence, by Bernstein's inequality \citep{BLM13}, with $\bbP_0$-probability at least $1 - \delta$,
\begin{align*}
\sum_{t=1}^n S_t &\leq \sqrt{-2 n \log(\epsilon) \log(1/\delta)} + \frac{3 \log(1/\delta)}{2 \sqrt{\epsilon}} = \tau \,.
\qedhere
\end{align*}
\end{proof}

\cref{lem:null} shows that under the null hypothesis the statistics are upper bounded with high probability.
The next lemma shows that under the alternative hypothesis they can be lower bounded with high probability.

\begin{lemma}\label{lem:S}
Suppose that $\epsilon = \log(1/\delta)/n$. 
Then, with $\bbP_1$-probability at least $1 - 2 \delta$,
\begin{align*}
\sum_{t=1}^n S_t \geq \frac{c}{2} \sum_{t=1}^n \Gl_\epsilon(P_t) - 3 \tau \,,
\end{align*}
where $c$ is the constant in \cref{lem:diff} and $\tau$ is defined in \cref{lem:null}.
\end{lemma}

\begin{proof}
Let $F_t(x) = \bbP_1(V_t \leq x | \cF_{t-1})$ and $W_t = F_t(V_t)$, which under
$\bbP_1(\cdot | \cF_{t-1})$ is uniformly distributed.
By \cref{lem:cdf}, $W_t \leq V_t$, which means that
\begin{align*}
R_t \triangleq \min\left(\frac{1}{\sqrt{\epsilon}}, \frac{1}{\sqrt{1 - W_t}}\right) - \mu \leq \min\left(\frac{1}{\sqrt{\epsilon}}, \frac{1}{\sqrt{1 - V_t}}\right) - \mu \triangleq S_t \,.
\end{align*}
Hence, $S_t - R_t$ is non-negative and $S_t - R_t \leq 1/\sqrt{\epsilon}$.
Therefore, by \cref{lem:conc}, with $\bbP_1$-probability at least $1 - \delta$,
\begin{align*}
\sum_{t=1}^n (S_t - R_t) 
&\exl{a}\geq \sum_{t=1}^n \E_1[S_t - R_t|\cF_{t-1}] - \sqrt{\epsilon}\sum_{t=1}^n \frac{\E_1\left[(S_t - R_t)^2|\cF_{t-1}\right]}{2} - \frac{2 \log(1/\delta)}{\sqrt{\epsilon}} \\
&\exl{b}\geq \sum_{t=1}^n \E_1[S_t - R_t|\cF_{t-1}] - \frac{1}{2} \sum_{t=1}^n \E_1\left[S_t - R_t|\cF_{t-1}\right] - \frac{2 \log(1/\delta)}{\sqrt{\epsilon}} \\
&= \frac{1}{2} \sum_{t=1}^n \E_1[S_t - R_t|\cF_{t-1}] - \frac{2 \log(1/\delta)}{\sqrt{\epsilon}} \\
&\exl{c}\geq \frac{c}{2} \sum_{t=1}^n \Gl_\epsilon(P_t) - \frac{cn \sqrt{\epsilon}}{2} - \frac{2  \log(1/\delta)}{\sqrt{\epsilon}} \\
&\exl{d}\geq \frac{c}{2} \sum_{t=1}^n \Gl_\epsilon(P_t) - \frac{cn \sqrt{\epsilon}}{2} - \tau \,,
\end{align*}
where \ex{a} follows from \cref{lem:conc},
\ex{b} since $S_t - R_t \in [0, 1/\sqrt{\epsilon}]$,
\ex{c} from \cref{lem:diff} and
\ex{d} from the definition of $\tau$.
The law of $R_t$ under $\bbP_1$ is the same as the law of $S_t$ under $\bbP_0$.
Hence, repeating the argument in \cref{lem:null} shows that
with $\bbP_1$-probability at least $1 - \delta$, $\sum_{t=1}^n R_t \geq -\tau$.
Combining the above with a union bound shows with probability at least $1 - 2 \delta$,
\begin{align*}
\sum_{t=1}^n S_t
&= \sum_{t=1}^n (S_t - R_t) + \sum_{t=1}^n R_t 
\geq \frac{c}{2} \sum_{t=1}^n \Gl_\epsilon(P_t) - \frac{cn\sqrt{\epsilon}}{2} - 2 \tau 
\geq \frac{c}{2} \sum_{t=1}^n \Gl_\epsilon(P_t) - 3\tau \,.
\qedhere
\end{align*}
\end{proof}

\begin{proof}[Proof of \cref{thm:upper}]
Let $C = 8/c$ where $c$ is defined in \cref{lem:diff} and
define events $E$ and $F$ by
\begin{align*}
E = \left\{\sum_{t=1}^n G_\epsilon(P_t) \geq C\tau \right\}  \quad \text{and} \quad
F = \left\{\sum_{t=1}^n S_t \geq \frac{c}{2} \sum_{t=1}^n G_\epsilon(P_t) - 3 \tau\right\} \,.
\end{align*}
By \cref{lem:S}, $\bbP(F) \geq 1 - 2\delta$.
Suppose that $E$ and $\sum_{t=1}^n S_t < \tau_\star$. Then
\begin{align*}
\tau \geq \tau_\star > \sum_{t=1}^n S_t \geq \sum_{t=1}^n S_t - \frac{c}{2} \sum_{t=1}^n G_\epsilon(P_t) + 4 \tau\,,
\end{align*}
which after arranging implies that $F^c$ holds.
Hence,
\begin{align*}
\bbP\left(\sum_{t=1}^n G_\epsilon(P_t) \geq C\tau,\, \text{\cref{alg:detect} returns \keyword{unknown}}\right)
&= \bbP\left(E \text{ and } \sum_{t=1}^n S_t < \tau_\star\right) 
\leq \bbP\left(F^c\right) 
\leq 2 \delta \,.
\qedhere
\end{align*}
\end{proof}

\cref{thm:upper} can be used to bound the detection time for 
various synthetic representations of a language model.
By definition $\tau = \Theta(\sqrt{n \log(n) \log(1/\delta)})$, so detection is likely to occur whenever
\begin{align*}
\sum_{t=1}^n G_\epsilon(P_t) = \Omega\left(\sqrt{n \log(n) \log(1/\delta)}\right) \,.
\end{align*}

\begin{example}[\textsc{no entropy}]
Suppose that $P_t$ is a Dirac for all $t$.
Then $G_\epsilon(P_t) = 0$ for all $t$ and detection does not occur with high probability.
Distortion-free watermarking is not possible without entropy.
\end{example}

\begin{example}[\textsc{constant low entropy}]\label{ex:const}
Suppose that $P_t = p$ where $p = (1-\beta, \beta)$ for some $\beta \in (0,1)$. Then
\begin{align}
\Gl_\epsilon(p) = \Theta\left(\sqrt{\beta(1-\beta)}\right) \,. 
\label{eq:ex:const}
\end{align}
Hence, by \cref{thm:upper}, detection occurs with probability at least $1 - 2 \delta$ if 
\begin{align*}
n G_\epsilon(p)
&= \sum_{t=1}^n G_\epsilon(P_t) = \Omega\left(\sqrt{n \log(n) \log(1/\delta)}\right)\,,
\end{align*}
which after substituting \cref{eq:ex:const} is equivalent to
$n = \Omega\left(\frac{\log(n) \log(1/\delta)}{\beta(1-\beta)}\right)$.
\end{example}

\begin{example}[\textsc{rare high entropy}]\label{ex:low}
Suppose that $P_t$ is a Dirac with probability $1 - \beta \in (0,1)$ and otherwise $P_t = p = (1/2,1/2)$.
Then $\Gl_\epsilon(p) = 1$ and with high probability
$\sum_{t=1}^n \Gl_\epsilon(P_t) = \Omega(\beta n)$ and detection occurs once 
\begin{align*}
n = \Omega\left(\frac{\log(n) \log(1/\delta)}{\beta^2}\right) \,.
\end{align*}
\end{example}

\begin{example}[\textsc{many rare tokens}]\label{ex:many}
Suppose that $d \geq 2$ and $\beta \in (0,1/2)$ and $P_t = p$ is constant with $p = (1 - \beta, \beta/(d-1), \beta/(d-1), \ldots, \beta/(d-1))$.
Then
\begin{align*}
G_\epsilon(p) &\approx \begin{cases}
\sqrt{d \beta} & \text{if } \beta/(d-1) \geq \epsilon \\
\sqrt{\beta} + \beta /\sqrt{\epsilon} & \text{otherwise} \,.
\end{cases}
\end{align*}
Then detection occurs once $n = \Omega\left(\frac{\log(n) \log(1/\delta)}{\beta}\right)$.
Note that this is roughly the time it takes to see $\Omega(\log(n) \log(1/\delta))$ rare tokens.
\end{example}

The previous examples are specialisations of the general setting that $(P_t)_{t=1}^n$ is sampled from
a product measure $\xi^n$.

\begin{example}[\textsc{iid setting}]\label{ex:iid}
Suppose that $P,P_1,\ldots,P_n$ are independent and identically distributed with law $\xi$ and $\Gl_\epsilon(P_t) = O(1)$ almost surely and that $\E[\Gh(P_t)] \ll 1$ is small, which is
the interesting case.
By Hoeffding's bound, with probability at least $1 - \delta$,
\begin{align*}
\sum_{t=1}^n \Gl_\epsilon(P_t) 
&\geq n \E[\Gl_\epsilon(P)] - O\left(\sqrt{n \log(1/\delta)}\right) \,.
\end{align*}
Hence detection occurs with high probability once
\begin{align}
\E[\Gl_\epsilon(P)] \approx \sqrt{\frac{\log(n) \log(1/\delta)}{n}} \,.
\label{eq:detect}
\end{align}
By definition,
\begin{align*}
\E\left[\Gl_\epsilon(P)\right] = \E\left[\Go(P_\epsilon) / \sqrt{\epsilon} + \Gh(P_\epsilon^c)\right] \geq \max\left(\E[\Go(P_\epsilon)] / \sqrt{\epsilon}, \E[\Gh(P_\epsilon^c)]\right) \,.
\end{align*}
Suppose that $n$ is the smallest value such that \cref{eq:detect} holds. 
By the definition of $\epsilon = \frac{\log(1/\delta)}{n}$,
\begin{align*}
n \approx \min\left(\frac{\sqrt{-\log(\E[\Go(P_\epsilon)])}}{\E[\Go(P_\epsilon)]},\, \frac{-\log(\E[\Gh(P_\epsilon^c)])}{\E[\Gh(P_\epsilon^c)]^2}\right) \log(1/\delta)
\end{align*}
This bound is not fully explicit, since $\epsilon$ on the right-hand side depends on $n$. Nevertheless, it serves as a useful comparison for our lower bound.
\end{example}

\section{Lower bounds}

\newcommand{\bbPn}{\mathbb P_0}
\newcommand{\bbPa}{\mathbb P_1}
\newcommand{\id}{\operatorname{id}}

The Gumbel watermarking procedure combined with the detection mechanism described above is nearly optimal in the setting where the detection process
cannot use the model. 
We will formalise this by assuming that $(P_t)_{t=1}^n$ is an independent 
and identically distributed sequence sampled from a suitably
symmetric distribution $\xi$ on $\Delta(\Sigma)$. 
Before this, however, we provide a more-or-less known lower bound
that holds even if the model is known.
This result is essentially known already, though with a different proof \citep{huang2023towards}.
The bound holds for any sequential watermarking scheme, which we now formalise.

\begin{definition}\label{def:valid}
A sequential watermarking scheme is characterised by a probability kernel $\kappa$ from $[0,1] \times \Delta(\Sigma) \to \Delta(\Sigma)$ such that for all $p \in \Delta(\Sigma)$,
\begin{align*}
\int_0^1 \kappa(\cdot | u, p) \d{u} = p \,. 
\end{align*}
Kernels $\kappa$ satisfying the above will be called valid selection kernels.
\end{definition}

This definition is intended to be used with secret source of noise $(U_t)_{t=1}^n$ that is i.i.d.\ with law $\cU([0,1])$.

\begin{example}
The Gumbel scheme can be recovered by letting $\kappa(\cdot | u, p)$ be a Dirac on
\begin{align*}
\argmax_{a \in \Sigma} \frac{p(a)}{-\log \phi_a(u)}\,,
\end{align*}
where $(\phi_a)_{a \in \Sigma} : [0,1] \to [0,1]$ is a collection of functions such that $(\phi_a(U))_{a \in \Sigma}$ has law $\cU([0,1]^\Sigma)$ whenever $U$ has law $\cU([0,1])$.
\end{example}

\begin{remark}
The kernel $\kappa(\cdot | u, p) = p$ for all $u \in [0,1]$ corresponds to ignoring the source of randomness. That is, no watermarking.
\end{remark}

\begin{assumption}\label{ass:lower-1}
Suppose that $(U_t)_{t=1}^n$ is an i.i.d.\ sequence with law $\cU([0,1])$ and $\kappa$ is a valid selection kernel (\cref{def:valid}).
Under the null distribution $\bbP_0$ we assume that $A_t$ has law $p(\cdot | A_1,\ldots,A_{t-1})$ and is independent of $(U_t)$.
Under the alternative distribution $\bbP_1$ it is assumed that $A_t$ is sampled from $\kappa(\cdot | U_t, P_t)$ where $P_t = p(\cdot | A_1,\ldots,A_{t-1})$.
\end{assumption}

Recall that $\cF_t = \sigma(A_1,U_1,\ldots,A_t,U_t)$ is the $\sigma$-algebra generated by the information available to the detector.

\begin{proposition}\label{prop:lower-simple}
Let $\bbP_0$ and $\bbP_1$ be the null and alternative measures defined in \cref{ass:lower-1}.
Suppose that $T \in \{0,1\}$ is $\cF_n$-measurable and
$\bbP_0(T = 0) \geq 1 - \delta$ and $\bbP_1(T = 1) \geq 1 - \delta$.
Then $n \geq \E[\bar H]^{-1} \log(1/(4\delta))$.
\end{proposition}

\begin{proof}
Let $I_t(X ; Y)$ be the mutual information between $X$ and $Y$ under $\bbP_1(\cdot | \cF_t)$. 
By the Bretagnolle--Huber inequality, 
\begin{align*}
2 \delta 
&\geq\bbPn(T = 1) + \bbPa(T = 0) 
\geq \frac{1}{2} \exp\left(-\KL(\bbP_1, \bbP_0)\right) \\
&= \frac{1}{2} \exp\left(-\E\left[\sum_{t=1}^n \KL(\bbP_{t-1}(U_t, A_t \in \cdot), \bbP_{t-1}(U_t \in \cdot) \otimes \bbP_{t-1}(A_t \in \cdot))\right]\right) \\
&= \frac{1}{2} \exp\left(-\E\left[\sum_{t=1}^n I_{t-1}(U_t; A_t)\right]\right) 
\geq \frac{1}{2} \exp\left(-\E\left[\sum_{t=1}^n H(P_t)\right]\right) 
= \frac{1}{2} \exp\left(-n \E[\bar H]\right) \,.
\end{align*}
Rearranging shows that $n \geq \E[\bar H]^{-1} \log(1/(4\delta))$.
\end{proof}

The Gumbel watermarking scheme combined with the new detection mechanism is nearly optimal when the model is not known and
when $(P_t)$ has law $\xi^n$ for some distribution $\xi$ on $\Delta(\Sigma)$. 
The assumption that the model is not known is expressed by making a symmetry assumption on $\xi$.

\begin{definition}
A probability measure $\xi$ on $\Delta(\Sigma)$ is symmetric if when $P$ has law $\xi$, then $P$ and $P \circ \sigma$ have the same law for any
permutation $\sigma$ of $\Sigma$.
\end{definition}

\begin{assumption}\label{ass:sym}
$\xi$ is a symmetric probability measure on $\Delta(\Sigma)$
and $\kappa$ is a valid selection kernel and $\kappa_0(\cdot | u, p) = p(\cdot)$. Let $\bbQ_0 = \lambda \otimes \xi \otimes \kappa_0$ and $\bbQ_1 = \lambda \otimes \xi \otimes \kappa$. 
Given $k \in \{0, 1\}$, let $\bbP_k$ be the law of $U_1,A_1,\ldots,U_n,A_n$ when $(U_t,P_t,A_t)$ has law $\bbQ_k^n$.
That is $\bbP_0$ is the law of $(U_1,A_1,\ldots,U_n,A_n)$ under the null hypothesis and $\bbP_1$ is the law of the same under the alternative hypothesis.
\end{assumption}

\begin{remark}
The symmetry assumption on $\xi$ ensures that $\bbP_1(A_t \in \cdot) = \cU(\Sigma)$. Hence the marginals of both $U_t$ and $A_t$ are the same
under $\bbP_0$ and $\bbP_1$. But for $\bbP_0$ they are independent, which is not true for $\bbP_1$.
\end{remark}

\begin{theorem}\label{thm:lower}
Suppose that $\xi$ and $\bbP_0$ and $\bbP_1$ are defined as in \cref{ass:sym} and that
$T \in \{0,1\}$ is $\cF_n$-measurable and
$\bbP_0(T = 0) \geq 1 - \delta$ and $\bbP_1(T = 1) \geq 1 - \delta$. 
Then, for any $\epsilon \in (0,1/2)$ and $P$ with law $\xi$,
\begin{align}
n \geq \min\left(\frac{1}{2 \E[\Go(P_\epsilon)]}, \, \frac{\log\left(\frac{1}{16\delta}\right)}{16 \E\left[\Gh(P_\epsilon^c)\right]^2}\right)\,.
\label{eq:lower}
\end{align}
\end{theorem}

\begin{proof}[Proof of \cref{thm:lower}]
We proceed in three steps.

\paragraph{Step 1: Setup and entropy inequalities}
Suppose that $\epsilon \in (0,1/2)$ and \cref{eq:lower} does not hold. 
Then $\E[\Go(P_\epsilon)] < \frac{1}{2n}$.
Let $E^n = \cap_{t=1}^n \{P_t(A_t) > \epsilon\}$.
Note that the law of $E^n$ is the same under both $\bbQ_0$ and $\bbQ_1$.
Let $P$ have law $\xi$ and $A \sim P$. Then,
\begin{align}
\bbQ^n_0(E^n)
&= \bbQ^n_0(P(A) > \epsilon)^n 
= (1 - \bbQ^n_0(P(A) \leq \epsilon))^n \nonumber \\
&= \left(1 - \E\left[\sum_{a \in \Sigma} P(a) \sind(P(a) \leq \epsilon)\right]\right)^n 
\geq \left(1 - 2\E\left[\Go(P_\epsilon)\right]\right)^n 
\geq \left(1 - \frac{1}{n}\right)^n \geq \frac{1}{4} \,, 
\label{eq:low}
\end{align}
where the second inequality follows from the assumption that $\E[\Go(P_\epsilon)] < 1/(2n)$ and the first
from the fact that $\epsilon \leq 1/2$ so that 
\begin{align*}
\Go(P_\epsilon) = \sum_{a : p(a) \leq \epsilon} p(a)(1-p(a)) 
\geq \sum_{a : p(a) \leq \epsilon} p(a) / 2 \,.
\end{align*}
From \cref{eq:low} it follows that $\bbQ^n_0(T = 1 | E^n) = \bbQ^n_0(T = 1, E^n) / \bbQ_0(E^n) \leq 4\delta$.
Similarly, $\bbQ^n_1(T = 0 | E^n) \leq 4\delta$.
Therefore, by the Bretagnolle--Huber inequality,
\begin{align*}
8 \delta 
&\geq \bbQ^n_0(T = 1 | E^n) + \bbQ^n_1(T = 0 | E^n) 
\geq \frac{1}{2} \exp\left(-\KL\right) \,,
\end{align*}
where $\KL = \KL(\bbQ_1^n(U_1,A_1,\ldots,U_n,A_n \in \cdot | E^n), \bbQ_0^n(U_1,A_1,\ldots,U_n,A_n \in \cdot | E^n))$.
Now $\bbQ^n_k(\cdot | E^n)$ is a product measure for both $k = 0$ and $k = 1$.
Therefore, abbreviating $U = U_1$ and $A = A_1$ and $\lambda = \cU([0,1])$ and $E = \{P(A) > \epsilon\}$,
\begin{align}
\KL
&= n \KL(\bbQ_1(U,A \in \cdot | E), \bbQ_0(U,A \in \cdot | E)) \nonumber \\
&= n \sum_{a \in \Sigma} \bbQ_1(A = a | E) \KL(\bbQ_1(U \in \cdot | A = a, E), \bbQ_0(U \in \cdot | A = a, E)) \nonumber \\
&= n \sum_{a \in \Sigma} \bbQ_1(A = a | E) \KL(\bbQ_1(U \in \cdot | A = a, E), \lambda)  \nonumber \\
&\leq \frac{n}{d} \sum_{a \in \Sigma} \chi^2(\bbQ_1(U \in \cdot | A = a, E), \lambda) \,,\
\label{eq:kl-1}
\end{align}
where the inequality follows from \cref{lem:chi} and because $\bbQ_1(A = a | E) = 1/d$ by the symmetry assumption.

\paragraph{Step 2: Bounding the chi-squared divergence}
We start by introducing a variety of conditional measures (regular versions):
\begin{itemize}
\item $\xi(\cdot|a) = \bbQ_1(P = \cdot | A = a, E)$.
\item $\mu(\cdot|a) = \bbQ_1(U = \cdot | A = a, E)$.
\item $\mu(\cdot|a, p) = \bbQ_1(U = \cdot | P, A = a, E)$ evaluated at $p$.
\end{itemize}
By Bayes' law, when $p(a) > \epsilon$, then 
\begin{align}
\frac{\d \xi(\cdot|a)}{\d \xi}(p) &= \frac{d p_\epsilon^c(a)}{\bbQ_1(E)}\,. \label{eq:change:xi} 
\end{align}
Next, if $p(a) > \epsilon$, then
\begin{align}
\frac{\d \mu(\cdot|a,p)}{\d \lambda}(x) = \frac{\kappa(a | x, p)}{\int_0^1 \kappa(a | y, p) \d{y}} \leq \frac{1}{p_\epsilon^c(a)} \,.
\label{eq:min}
\end{align}
Let $Q$ be independent of $P$ and have law $\xi$.
Then:
\begin{align*}
\chi^2(\mu(\cdot | a), \lambda)
&= \int_0^1 \left(\frac{\d{\mu(\cdot | a)}}{\d{\lambda}}(x) - 1\right)^2 \d{x} \\
&= \int_0^1 \frac{\d{\mu(\cdot|a)}}{\d{\lambda}}(x)^2 \d{x} - 1 \\
&\exl{a}= \int_0^1 \left(\int_{\Delta(\Sigma)} \frac{\d{\mu(\cdot|a,p)}}{\d{\lambda}}(x) \d{\xi(p|a)}\right)^2 \d{x} - 1 \\
&= \int_{\Delta(\Sigma)} \int_{\Delta(\Sigma)} \left[\int_0^1 \frac{\d{\mu(\cdot|a,p)}}{\d{\lambda}}(x) \frac{\d{\mu(\cdot|a,q)}}{\d{\lambda}}(x) \d{x}\right] \d{\xi(p|a)} \d{\xi(q|a)} - 1 \\
&\exl{b}\leq \int_{\Delta(\Sigma)}\int_{\Delta(\Sigma)} \min\left(\frac{1}{p_\epsilon^c(a)}, \frac{1}{q_\epsilon^c(a)}\right) \d{\xi(p|a)} \d{\xi(q|a)} - 1 \\ 
&= \int_{\Delta(\Sigma)}\int_{\Delta(\Sigma)} \min\left(\frac{1 - p_\epsilon^c(a)}{p_\epsilon^c(a)}, \frac{1-q_\epsilon^c(a)}{q_\epsilon^c(a)}\right) \d{\xi(p|a)} \d{\xi(q|a)}  \\ 
&\exl{c}= \frac{d^2}{\bbQ_1(E)^2} \int_{\Delta(\Sigma)} \int_{\Delta(\Sigma)} \min\left(\frac{1-p_\epsilon^c(a)}{p_\epsilon^c(a)}, \frac{1-q_\epsilon^c(a)}{q_\epsilon^c(a)}\right) p_\epsilon^c(a) q_\epsilon^c(a) \d{\xi(p)} \d{\xi(q)}  \\
&= \frac{d^2}{\bbQ_1(E)^2} \E\left[\min\left(Q_\epsilon^c(a)(1-P_\epsilon^c(a)), P_\epsilon^c(a)(1-Q_\epsilon^c(a))\right)\right] \\
&\exl{d}\leq 16d^2 \E\left[\min\left(P_\epsilon^c(a)(1 - P_\epsilon^c(a)),\, Q_\epsilon^c(a)(1 - Q_\epsilon^c(a))\right)\right] \\
&\exl{e}\leq 16d^2 \E\left[\sqrt{P_\epsilon^c(a)(1 - P_\epsilon^c(a))} \right]^2 \,.
\end{align*}
where \ex{a} follows from the definition of $\mu(\cdot|a)$, $\mu(\cdot|a,p)$ and $\xi(p|a)$,
\ex{b} from \cref{eq:min},
\ex{c} from \cref{eq:change:xi},
\ex{d} since $\min(p(1-q), q(1-p)) \leq \min(p(1-p), q(1-q))$ and $\bbQ_0(E) \geq \bbQ_0^n(E^n) \geq 1/4$ by \cref{eq:low}. Finally,
\ex{e} holds because $\min(p(1-p), q(1-q)) \leq \sqrt{p(1-p)q(1-q)}$.

\paragraph{Step 3: Finishing up}

Combining the result in the previous step with \cref{eq:kl-1},
\begin{align*}
\KL
&\leq 16nd \sum_{a=1}^d \E\left[\sqrt{P_\epsilon^c(a)(1 - P_\epsilon^c(a))}\right]^2 
= 16n \E\left[\sum_{a=1}^d \sqrt{P_\epsilon^c(a)(1 - P_\epsilon^c(a))}\right]^2 
= 16n \E\left[\Gh(P_\epsilon^c)\right]^2 \,.
\end{align*}
Therefore
$16\delta \geq \exp(-16n \E[\Gh(P_\epsilon^c)]^2)$,
which implies that
\begin{align*}
n \geq \frac{\log\left(\frac{1}{16\delta}\right)}{16 \E\left[\Gh(P_\epsilon^c)\right]^2}\,,
\end{align*}
which contradicts our assumption that \cref{eq:lower} does not hold.
\end{proof}

\newcommand{\Nupper}{N_{\textsc{upper}}}
\newcommand{\Nlower}{N_{\textsc{lower}}}
\newcommand{\Nopt}{N_{\textsc{opt}}}

\section{Comparing upper and lower bounds}
The upper and lower bounds are nearly matching in the i.i.d.\ setting of \cref{ex:iid}.
The upper bound shows that detection occurs with high probability with Gumbel watermarking and power law detection after observing $n \geq \Nupper$ tokens where
$\Nupper$ is the smallest natural number such that
\begin{align*}
\Nupper \approx \min\left(\frac{\sqrt{-\log \E[\Go(P_\epsilon)]}}{\E[\Go(P_\epsilon)]}, \frac{-\log \E[\Gh(P_\epsilon^c)]}{\E[\Gh(P_\epsilon^c)]^2}\right) \log(1/\delta)
\end{align*}
where $\epsilon = \log(1/\delta) / \Nupper \leq 1/2$ with the inequality holding by the assumption in \cref{ex:iid} that $\E[\Gh(P)]$ is sufficiently small.
The lower bound (\cref{thm:lower}) shows that detection cannot occur unless $n \geq \Nlower$ where
\begin{align*}
\Nlower
\approx \min\left(\frac{1}{\E[\Go(P_\epsilon)]},\, \frac{\log(1/\delta)}{\E[\Gh(P_\epsilon^c)]^2}\right) \,.
\end{align*}
Note that the lower bound holds for any $\epsilon \in (0,1/2)$, and in particular for $\epsilon = \log(1/\delta) / \Nupper$.
Hence the upper and lower bounds are separated by logarithmic factors.

\section{Discussion}\label{sec:disc}

\paragraph{Experiments}
You surely noticed the absence of experiments.
Our brief tests showed the power law detector indeed outperforms the exponential detector in synthetically crafted examples where 
the theory says it should (e.g., \cref{ex:const}).
Unfortunately, however, its performance on language data seems to be slightly worse, presumably due to the additional logarithmic factor and/or constant factors.
There may be ways to improve the proposed statistic to get the best of both worlds. Most naively one can combine both detectors with a union bound, but better methods
might be possible.

\paragraph{Detecting non-watermarked text}
\cref{alg:detect} either rejects the null hypothesis and returns \keyword{watermarked} or returns
\keyword{unknown}. The latter result does not correspond to a high-probability guarantee that the text is not
watermarked. Such an assertion is not possible in the standard frequentist setting considered here because the signal
provided by watermarking can be arbitrarily small when the entropy of the conditional distributions $(P_t)$ is small.
However, if there is a known (probabilistic) lower bound on $\sum_{t=1}^n G_\epsilon(P_t)$ under $\bbP_1$, then it is possible
to accept the null hypothesis if $\sum_{t=1}^n S_t$ does not grow suitably fast.

\paragraph{Intuition for statistic}
The intuition behind the definition of $S$ in \cref{eq:S} is as follows.
The key statistical approach is based on testing whether or not $(V_t)_{t=1}^n$ could reasonably be sampled from
a uniform distribution, which falls into the goodness-of-fit class of statistical tests \citep{d2017goodness}.
Many of these tests are based on some kind of comparison between the empirical CDF of $(V_t)$ and the CDF of the uniform
distribution, which for $x \in [0,1]$ is $F(x) = x$. The empirical CDF is $F_n(x) = \frac{1}{n} \sum_{t=1}^n \sind(V_t \leq x)$.
Classical statistics include the Kolmogorov--Smirnov and Anderson--Darling statistics:
\begin{align*}
\sup_{x \in [0,1]} \left(F(x) - F_n(x)\right)^2 \quad \text{and} \quad \int_0^1 \left(\frac{F(x) - F_n(x)}{\sqrt{x(1-x)}}\right)^2 \d{x} \,.
\end{align*}
These statistics cannot be written as sums of functionals of $V_t$. Moreover, the Kolmogorov--Smirnov statistic is not making
use of the fact that under the null hypothesis the variance of $F(x) - F_n(x)$ is $x(1-x)/n$.
This suggests a weighted Kolmogorov--Smirnov statistic:
\begin{align}
\sup_{x \in (0,1)} \frac{F(x) - F_n(x)}{\sqrt{x(1-x)}} \,,
\label{eq:wks}
\end{align}
which was introduced and analysed by \cite{AD52} as a goodness-of-fit test. \cite{DJ04} also studied this statistic in the context of multiple hypotheses
as we do here, but with a less complex structure for the alternative.
As it turns out, this statistic would yield about the same theory as our choice.
The statistic used here is based on the observation that $F(x) - F_n(x)$ is expected to be large for values of $x \in [1/2,1]$,
so the $\sqrt{x}$ term in the denominator has little impact. Moreover, rather than taking a supremum we average `uniformly' on buckets
$[0,1/2],[1/2,3/4],[3/4,7/8],\ldots$, which corresponds to
\begin{align*}
\int_0^{1-\epsilon} \frac{F(x) - F_n(x)}{(1-x)^{3/2}} \d{x} = \frac{2}{n} \sum_{t=1}^n \min\left(\frac{1}{\sqrt{\epsilon}}\,, \frac{1}{\sqrt{1 - V_t}}\right) - 4 + 2\sqrt{\epsilon} \,,
\end{align*}
which except for a scale and a shift is the proposed statistic.
The truncation is needed to ensure the statistic has well-controlled tails under the null hypothesis.

\paragraph{Weighted detectors}
The proposed detection mechanism is entirely independent of the model.
Detection can be improved enormously given access to a low-fidelity surrogate model.
Consider the situation in \cref{ex:low} where most tokens are nearly deterministic.
If the detector has access to a classifier $\phi : \Sigma^* \to \{0, 1\}$ such that
$\phi(A_1,\ldots,A_{t-1}) = 0$ if $P_t$ is close to a Dirac, then a better test statistic
is $\sum_{t=1}^n \phi(A_1,\ldots,A_{t-1})S_t$, which can then be compared to the critical
value for sample size $\sum_{t=1}^n \phi(A_1,\ldots,A_{t-1})$.
Assuming that $\phi$ is accurate, then this approach reduces the detection time for \cref{ex:low} from $\tilde O(1/\beta^2)$ to $\tilde O(1/\beta)$.
Alternatively, $\phi$ can be real-valued (and even learned). Concretely, suppose that $\phi : \Sigma^* \to [0,1]$ and detection has the form
\begin{align*}
T = \sind\left(\sum_{t=1}^n \phi(A_1,\ldots,A_{t-1}) S_t \geq \tau_\star\right)\,,
\end{align*}
where $S_t = S(V_t)$ is defined as in \cref{eq:S}.
Since $\alpha_t = \phi(A_1,\ldots,A_{t-1})$ only depends on $(A_t)$, under the null hypothesis $(\alpha_t)$ and $(S_t)$ are independent. Hence, you can
still use Monte-Carlo to calculate a data-dependent critical threshold $\tau_\star$ such that
\begin{align*}
\bbP\left(\sum_{t=1}^n \alpha_t S_t \geq \tau_\star \bigg|(\alpha_t)_{t=1}^n\right) = \delta\,.
\end{align*}
Theoretically, by repeating the argument in \cref{lem:null}, with $\bbP_0$-probability at least $1 - \delta$,
\begin{align*}
\sum_{t=1}^n \alpha_t S_t = \tilde O\left(\sqrt{\sum_{t=1}^n \alpha_t^2} + \frac{1}{\sqrt{\epsilon}}\right) \,.
\end{align*}
Repeating also the proof of \cref{lem:S} shows that with $\bbP_1$-probability at least $1 - \delta$,
\begin{align*}
\sum_{t=1}^n \alpha_t S_t = \Omega\left(\sum_{t=1}^n \alpha_t G_\epsilon(P_t)\right) - \tilde O\left(\frac{1}{\sqrt{\epsilon}} + \sqrt{\sum_{t=1}^n \alpha_t^2}\right) \,.
\end{align*}
This is a complicated expression to optimise because $G_\epsilon(p)$ depends on $\epsilon$. A reasonable approximation in many practical scenarios is
$G_\epsilon(p) \approx \Gh(p)$ for some constant $\epsilon$ (e.g., $\epsilon = 1/10$). Whenever this approximation is reasonable and $\alpha_t = \Gh(P_t)$, then detection occurs once
\begin{align*}
\sum_{t=1}^n \Gh(P_t)^2 = \tilde \Omega(1) \,.
\end{align*}
Of course, this choice of $\alpha_t$ depends on $\Gh(P_t)$ that is generally not known. An interesting question is how to choose $\alpha_t$ robustly when $\Gh(P_t)$ can only
be estimated; and also in the case that $G_\epsilon(P_t) \approx \Gh(P_t)$ is a poor approximation.

\paragraph{Asymptotics of the critical statistic}
The Lindeberg--Feller central limit theorem provides an asymptotic approximation of $\tau_\star$ defined in \cref{eq:opt}.
When $U$ has law $\cU([0,1])$, then
\begin{align*}
\E[S(U)] = 0 \quad \text{and} \quad
\E[S(U)^2] = \log(1/\eps) - O(1) \,.
\end{align*}
Hence, with $\eps = \log(1/\delta) / n$ the central limit theorem approximation of $\tau_\star$ for large $n$ is
\begin{align*}
\tau_\star \sim \Phi^{-1}(1 - \delta) \sqrt{n \log n} \,,
\end{align*}
where $\Phi$ is the cumulative distribution function of the standard Gaussian.

\paragraph{Non-optimality of subexponential statistics}
A random variable $X$ is called $\sigma$-subexponential if $\E[\exp(|X/\sigma|)] \leq 2$ \citep{Ver18}.
Our choice of $S_t$ is a truncated power law under the null hypothesis, which is not subexponential for any $\sigma = O(1)$.
On the other hand, \cite{AAR23} chose $S_t$ to be exponentially distributed
and \cite{dathathri2024scalable} essentially chose the statistic to be Gaussian (more precisely, a convolution of Bernoulli's).
Both of these choices are $\sigma$-subexponential with $\sigma = \Theta(1)$.
Suppose that the statistics $S_t$ are scaled and shifted so that under the null hypothesis $\E_0[S_t] = 0$ and $\E_0[S_t^2] = 1$ and $\E_0[\exp(|S_t/\sigma|)] \leq 2$
for some $\sigma = O(1)$.
We argue under these conditions that using a detector of the form
\begin{align*}
T = \sind\left(\sum_{t=1}^n S_t \geq \tau\right)
\end{align*}
cannot lead to near-optimal detection.
To see why, under the alternative hypothesis, using \cref{lem:fench},
\begin{align*}
\E_1[S_t|P_t] 
= \sum_{a \in \Sigma} \E_1[S_t \sind(A_t = a)|P_t] 
\leq \frac{5}{2} \sigma \sum_{a \in \Sigma} h(P_t(a)) 
\leq 5 \sigma H(P_t)\,,
\end{align*}
where $h(p) = -p \log(p) - (1-p) \log(1-p)$ is the binary entropy.
Since we assumed that $\E_0[S_t] = 0$ and $\E_0[S_t^2] = 1$, the central limit theorem suggests the critical threshold satisfies $\tau_\star = \tilde \Theta(\sqrt{n})$,
it follows that detection occurs only when $\sigma \sum_{t=1}^n H(P_t) = \tilde \Omega(\sqrt{n})$, which when $\sigma = O(1)$ implies that $n = \tilde \Omega(1/\bar H^2)$.
We showed in \cref{ex:const}, however, that the power law statistic yields detection once
$n = \tilde \Omega(1 / \bar H)$.
\cite{tsur2025heavywater} also noted the importance of using a heavier-tailed statistic for more efficient detection.

\paragraph{Logarithmic factors}
There is a logarithmic factor gap between the upper and lower bounds.
The upper bound is probably not improvable for the proposed power law detection scheme.
In certain cases (see \cref{ex:low}) the bound on the detection time
for the newly proposed detection scheme is worse than the bound provided by \cite{AAR23}.
Since the watermarking scheme is the same for both, one can of course combine the tests proposed here
with that suggested by \cite{AAR23} at the cost of a union bound.

\paragraph{Acknowledgements}
Many thanks to Marcus Hutter and Andr\'as Gy\"orgy for useful discussion and suggestions.

\bibliographystyle{abbrvnat}
\bibliography{all}

@article{huang2023towards,
  title={Towards optimal statistical watermarking},
  author={Huang, Baihe and Zhu, Hanlin and Zhu, Banghua and Ramchandran, Kannan and Jordan, Michael I and Lee, Jason D and Jiao, Jiantao},
  journal={arXiv preprint arXiv:2312.07930},
  year={2023}
}

@article{DJ04,
author = {David Donoho and Jiashun Jin},
title = {{Higher criticism for detecting sparse heterogeneous mixtures}},
volume = {32},
journal = {The Annals of Statistics},
number = {3},
publisher = {Institute of Mathematical Statistics},
pages = {962 -- 994},
year = {2004},
}

@article{ji2025overview,
  title={An overview of large language models for statisticians},
  author={Ji, Wenlong and Yuan, Weizhe and Getzen, Emily and Cho, Kyunghyun and Jordan, Michael I and Mei, Song and Weston, Jason E and Su, Weijie J and Xu, Jing and Zhang, Linjun},
  journal={arXiv preprint arXiv:2502.17814},
  year={2025}
}

@article{li2025robust,
  title={Robust detection of watermarks for large language models under human edits},
  author={Li, Xiang and Ruan, Feng and Wang, Huiyuan and Long, Qi and Su, Weijie J},
  journal={Journal of the Royal Statistical Society Series B: Statistical Methodology},
  pages={qkaf056},
  year={2025},
  publisher={Oxford University Press UK}
}

@article{huang2026towards,
  title={Towards Anytime-Valid Statistical Watermarking},
  author={Huang, Baihe and Xu, Eric and Ramchandran, Kannan and Jiao, Jiantao and Jordan, Michael I},
  journal={arXiv preprint arXiv:2602.17608},
  year={2026}
}

@inproceedings{tsur2025heavywater,
  title={HeavyWater and SimplexWater: Distortion-free LLM Watermarks for Low-Entropy Distributions},
  author={Tsur, Dor and Long, Carol Xuan and Verdun, Claudio Mayrink and Vithana, Sajani and Hsu, Hsiang and Chen, Chun-Fu and Permuter, Haim H and Calmon, Flavio},
  booktitle={The Thirty-ninth Annual Conference on Neural Information Processing Systems},
  year={2025}
}

@article{he2025empirical,
  title={On the Empirical Power of Goodness-of-Fit Tests in Watermark Detection},
  author={He, Weiqing and Li, Xiang and Shang, Tianqi and Shen, Li and Su, Weijie and Long, Qi},
  journal={arXiv preprint arXiv:2510.03944},
  year={2025}
}

@article{AD52,
  title={Asymptotic theory of certain" goodness of fit" criteria based on stochastic processes},
  author={Anderson, Theodore W and Darling, Donald A},
  journal={The annals of mathematical statistics},
  pages={193--212},
  year={1952},
  publisher={JSTOR}
}

@book{d2017goodness,
  title={Goodness-of-fit-techniques},
  author={D'Agostino, RalphB},
  year={2017},
  publisher={Routledge}
}

@book{Ver18,
  author =        {R. Vershynin},
  publisher =     {Cambridge University Press},
  title =         {High-dimensional probability: An introduction with
                   applications in data science},
  year =          {2018},
}

@article{liu2024survey,
  title={A survey of text watermarking in the era of large language models},
  author={Liu, Aiwei and Pan, Leyi and Lu, Yijian and Li, Jingjing and Hu, Xuming and Zhang, Xi and Wen, Lijie and King, Irwin and Xiong, Hui and Yu, Philip},
  journal={ACM Computing Surveys},
  volume={57},
  number={2},
  pages={1--36},
  year={2024},
  publisher={ACM New York, NY}
}

@book{Tsy08,
	Author = {Tsybakov, A. B.},
	Publisher = {Springer Science \& Business Media},
	Title = {Introduction to nonparametric estimation},
	Year = {2008}}

@book{LS20book,
  title={Bandit algorithms},
  author={T. Lattimore and Cs.\ Szepesv{\'a}ri},
  year={2020},
  publisher={Cambridge University Press}
}

@article{li2025statistical,
  title={A statistical framework of watermarks for large language models: Pivot, detection efficiency and optimal rules},
  author={Li, Xiang and Ruan, Feng and Wang, Huiyuan and Long, Qi and Su, Weijie J},
  journal={The Annals of Statistics},
  volume={53},
  number={1},
  pages={322--351},
  year={2025},
  publisher={Institute of Mathematical Statistics}
}

@book{BLM13,
	Author = {Boucheron, S. and Lugosi, G. and Massart, P.},
	Publisher = {OUP Oxford},
	Title = {Concentration inequalities: A nonasymptotic theory of independence},
	Year = {2013}}

@article{AAR23,
  author = {S. Aaronson},
  title={Watermarking {LLM}s},
  journal={online},
  year = {2022},
  url = {https://www.scottaaronson.com/talks/watermark.ppt}
}

@inproceedings{fernandez2023three,
  title={Three bricks to consolidate watermarks for large language models},
  author={Fernandez, Pierre and Chaffin, Antoine and Tit, Karim and Chappelier, Vivien and Furon, Teddy},
  booktitle={2023 IEEE international workshop on information forensics and security (WIFS)},
  pages={1--6},
  year={2023},
  organization={IEEE}
}

@article{tsur2025optimized,
  title={Optimized Couplings for Watermarking Large Language Models},
  author={Tsur, Dor and Long, Carol Xuan and Verdun, Claudio Mayrink and Hsu, Hsiang and Permuter, Haim and Calmon, Flavio P},
  journal={arXiv preprint arXiv:2505.08878},
  year={2025}
}

@article{liu2023unforgeable,
  title={An unforgeable publicly verifiable watermark for large language models},
  author={Liu, Aiwei and Pan, Leyi and Hu, Xuming and Li, Shu'ang and Wen, Lijie and King, Irwin and Yu, Philip S},
  journal={arXiv preprint arXiv:2307.16230},
  year={2023}
}

@article{zhao2023provable,
  title={Provable robust watermarking for ai-generated text},
  author={Zhao, Xuandong and Ananth, Prabhanjan and Li, Lei and Wang, Yu-Xiang},
  journal={arXiv preprint arXiv:2306.17439},
  year={2023}
}

@inproceedings{piet2025markmywords,
  title={MARKMyWORDS: Analyzing and Evaluating Language Model Watermarks},
  author={Piet, Julien and Sitawarin, Chawin and Fang, Vivian and Mu, Norman and Wagner, David},
  booktitle={2025 IEEE Conference on Secure and Trustworthy Machine Learning (SaTML)},
  pages={68--91},
  year={2025},
  organization={IEEE}
}

@article{kuditipudi2023robust,
  title={Robust distortion-free watermarks for language models},
  author={Kuditipudi, Rohith and Thickstun, John and Hashimoto, Tatsunori and Liang, Percy},
  journal={arXiv preprint arXiv:2307.15593},
  year={2023}
}

@InProceedings{kirchenbauer23a,
  title = 	 {A Watermark for Large Language Models},
  author =       {Kirchenbauer, John and Geiping, Jonas and Wen, Yuxin and Katz, Jonathan and Miers, Ian and Goldstein, Tom},
  booktitle = 	 {Proceedings of the 40th International Conference on Machine Learning},
  pages = 	 {17061--17084},
  year = 	 {2023},
  editor = 	 {Krause, Andreas and Brunskill, Emma and Cho, Kyunghyun and Engelhardt, Barbara and Sabato, Sivan and Scarlett, Jonathan},
  volume = 	 {202},
  series = 	 {Proceedings of Machine Learning Research},
  month = 	 {23--29 Jul},
  publisher =    {PMLR},
}

@inproceedings{christ2024undetectable,
  title={Undetectable watermarks for language models},
  author={Christ, Miranda and Gunn, Sam and Zamir, Or},
  booktitle={The Thirty Seventh Annual Conference on Learning Theory},
  pages={1125--1139},
  year={2024},
  organization={PMLR}
}

@article{dathathri2024scalable,
  title={Scalable watermarking for identifying large language model outputs},
  author={Dathathri, Sumanth and See, Abigail and Ghaisas, Sumedh and Huang, Po-Sen and McAdam, Rob and Welbl, Johannes and Bachani, Vandana and Kaskasoli, Alex and Stanforth, Robert and Matejovicova, Tatiana and others},
  journal={Nature},
  volume={634},
  number={8035},
  pages={818--823},
  year={2024},
  publisher={Nature Publishing Group UK London}
}

\appendix

\section{Technical inequalities}

\begin{lemma}\label{lem:fench}
Suppose that $\sigma > 0$ and $X$ is a random variable with $\E[X] = 0$ and $\E[\exp(|X/\sigma|)] \leq 2$ and $E \in \{0,1\}$ be
such that $\bbP(E = 1) = p$. Then 
\begin{align*}
\E[XE] \leq \frac{1 + \log(2)}{\log(2)} \sigma h(p) \,.
\end{align*}
where $h(p) = -p \log p - (1-p) \log(1-p)$ is the binary entropy.
\end{lemma}

\begin{proof}
Let $f(x) = x \log x - x$, which has Fenchel dual $f^\star(u) = \exp(u)$.
Hence, by Fenchel--Young's inequality,
\begin{align*}
\E[X E]
&\leq \E[|X| E] 
= p\sigma \E[|X/\sigma| E p] 
\leq p \sigma \left(\E[f^\star(|X/\sigma|)] + \E[f(E/p)]\right) 
\leq p \sigma \left(1 + \log(1/p)\right) \,.
\end{align*}
On the other hand, $\E[X E] = \E[X(E - 1)] \leq \E[|X|(1 - E)]$ and hence the same argument shows that
$\E[XE] \leq (1 - p) \sigma \left(1 + \log(1/(1-p))\right)$.
Therefore
\begin{align*}
\E[XE] 
&\leq \sigma\min\left(p(1 + \log(1/p)),\, (1 - p) (1 + \log(1/(1-p)))\right)  
\leq \frac{1 + \log(2)}{\log(2)} \sigma h(p) \,.
\qedhere
\end{align*}
\end{proof}

\begin{lemma}[\citealt{Tsy08}]\label{lem:chi}
Let $p$ and $q$ be probability measures with $p$ absolutely continuous with respect to $q$.
Then $\KL(p, q) \leq \chi^2(p,q)$.
\end{lemma}

\begin{lemma}[\citealt{LS20book}]\label{lem:conc}
Suppose that $(X_t)_{t=1}^n$ is a sequence of random variables adapted to a filtration $(\sF_t)_{t=1}^n$
and $\eta > 0$ is a constant such that $\eta |X_t|\leq 1$ almost surely.
Then, with probability at least $1 - \delta$,
\begin{align*}
\sum_{t=1}^n X_t \leq \sum_{t=1}^n \E[X_t|\sF_{t-1}] + \eta \sum_{t=1}^n \E[X_t^2|\sF_{t-1}] + \frac{\log(1/\delta)}{\eta} \,.
\end{align*}
\end{lemma}

Lastly we supply for completeness a proof of \cref{lem:cdf}.

\begin{proof}[Proof of \cref{lem:cdf}]
By definition, 
\begin{align*}
\bbP(V \leq x) 
&= \sum_{a=1}^d \bbP(U_a \leq x, A = a) 
= \sum_{a=1}^d \int_0^x \bbP\left(\frac{p(b)}{-\log(U_b)} < \frac{p(a)}{-\log(y)} \text{ for all } b \neq a\right) \d{y} \\
&= \sum_{a=1}^d \int_0^x \bbP\left(U_b < y^{p(b)/ p(a)} \text{ for all }b \neq a\right) \d{y} 
= \sum_{a=1}^d \int_0^x \prod_{b \neq a} y^{p(b) / p(a)} \d{y} \\
&= \sum_{a=1}^d \int_0^x y^{\sum_{b \neq a} p(b) / p(a)} \d{y}  
= \sum_{a=1}^d \int_0^x y^{1/p(a) - 1} \d{y} 
= \sum_{a=1}^d p(a) x^{1/p(a)} \,.
\qedhere
\end{align*}
\end{proof}

\end{document}